\documentclass{llncs}
\input{scrload.tex}
\usepackage{amssymb}

\def\squareforqed{\hbox{\rlap{$\sqcap$}$\sqcup$}}
\def\qed{\ifmmode\squareforqed\else{\unskip\nobreak\hfil
\penalty50\hskip1em\null\nobreak\hfil\squareforqed
\parfillskip=0pt\finalhyphendemerits=0\endgraf}\fi}
\renewcommand{\emptyset}{\varnothing}

\begin{document}

\titlerunning{Reasoning about unpredicted change and explicit time}
\title{Reasoning about unpredicted change and explicit time}

\author{Florence Dupin de Saint-Cyr and J\'er\^ome Lang}
\authorrunning{F. Dupin de Saint-Cyr and J. Lang}
\tocauthor{Florence Dupin de Saint-Cyr, J\'er\^ome Lang (IRIT)}
\institute{IRIT - Universit\'e Paul Sabatier - \\
31062 Toulouse Cedex(France)\\
e-mail: \{dupin, lang\} @irit.fr}

\maketitle
\bibliographystyle{plain}

\begin{abstract}
Reasoning about unpredicted change consists in explaining observations by events; we
propose here an approach for explaining time-stamped observations by surprises, which
are simple events consisting in the change of the truth value of a fluent. A framework for 
dealing with
surprises is defined. Minimal
sets of surprises are provided together with time intervals where each surprise has 
occurred, and they are characterized from a model-based diagnosis point of view. 
Then, a probabilistic approach of surprise minimisation is proposed.
\end{abstract} 


\section{Introduction}
Reasoning about time, action and change in full generality is a very complex task, and
much research has been done in order to find solutions suited to specific subclasses of
problems, obtained by making some simplificaying assumptions. Sandewall
\cite{Sandewall94}
provides a taxonomy for positioning the various subclasses, based on a lot of
ontological and epistemological assumptions that may be present or absent in the
specification of the subclass. In this paper we consider an integer representation of
time and as in \cite{Sandewall94} we define {\em fluents} as propositions whose truth value
evolves over time, {\em observations} as pieces of knowledge about the value of some
fluents at some time points, {\em changes} as pairs made of a fluent $f$ and a time
point $t$ such that the value of the fluent $f$ at $t$ differs from that at $t-1$.
Changes may be either caused by {\em actions}, which are initiated by the agent, or by
{\em events}, which are initiated by the world. This latter form of change will be
called {\em unpredicted changes}; as they are unrelated to the performance of any
action by the agent, they are totally out of the control of the agent. 
While many approaches in the literature
focus on reasoning about change caused by actions\footnote{This is for instance the case 
for all approaches to solve the frame problem.},
there has been significantly less attention on reasoning about unpredicted change.

In this article we propose a framework for detecting unpredicted changes from 
observations at different time points: the task consists thus in {\em
explaining} observations by a set of elementary unpredicted changes.
Note that it is analogous in many aspects to temporal diagnosis (an elementary change
corresponding to a faulty component). We do not explicitly consider actions, since
this is not our purpose; however, as we will show, our framework can handle actions
as well, provided that they have a deterministic, unconditional result -- thus the
class of problems we consider is ${\cal K}-IS$ of Sandewall's taxonomy (see Section
2). Then, we propose some qualitative and quantitative criteria for ranking 
explanations. We propose a formal basis for handling problems in the class K-IS; namely, we 
argue that what has to be computed is the set of {\em minimal compact explanations}, where 
minimality means that these explanations do not contain any unnecessary change, and compactness
means that they are given as concisely, as user-friendly as possible. Then we focus on computational
issues and for this purpose we relate our framework to (temporal) model-based diagnosis: a problem of
K-IS is translated into a system to diagnose and we show that minimal explanations correspond
to minimal fault configurations. Thus, computational tools from model-based diagnosis can be reused for
reasoning about unpredicted change. In Section 3, we
compute a probability distribution on these explanations, which takes account 
of (i) the intrinsic tendency of each individual fluent to persist, and
(ii) the time interval during which the change may have taken place -- following the
principle that the longer the time period since we last knew the value of a fluent,
the fewer reasons we have to believe that this value has persisted. 
We illustrate this on a small example: 
{\em two cats live in a house, namely, Albert and Bart. 
At 6 p.m. , somebody told me
that cat A. was sleeping inside the house. At 6h15 p.m., I knew that cat B.
was sleeping inside the house. Now, it is 6h20 p.m., I can see a cat walking in the 
garden, but I can't recognize from where I am, if it is Albert or Bart.}
The two intuitive minimal explanations we wish to get are the following: 
either
{\em cat A. left the house between 6p.m. and  6h20 p.m.}, or {\em cat B. left
the house between 6h15p.m. and 6h20}. Now, imagine we know the (subjective)
prior probabilities for cat A. and cat B. to leave the house in a time unit (one minute); 
from these, and from the observations, we may compute the probabilities of all
explanations. Intuitively, if the probabilities relative to cat A. and cat B. are
identical, then the explanation {\em cat A. got outside} is preferred to
{\em cat B. got outside}, since more time has passed between the last 
information we got and now -- thus, a change is more likely to have occurred.
However, the result may not be the same if the prior probability of cat A. getting
outside is much lower than cat B's (cat A is lazier than B). What we propose in Section 3 
is a simple method
to compute most probable explanations in the case where probabilities of change are small 
within the time period considered, with a methodology similar to computing
preferred diagnoses. We end up by positioning our approach w.r.t. related work and 
mention possible further extension.

\section{Minimizing unpredicted change in K-IS}

\subsection{Background}

From now on, the class of problems we consider coincides with the class ${\cal
K}-IS$ in Sandewall's taxonomy. The symbol ${\cal K}$ means that all
observations reported to the agent are correct, {\em i.e.}, true in the real
world. $I$ is for {\em inertia with integer time} and means that 
fluent values generally tend to persist. $S$ is for {\em surprises}; a
surprise, or unpredicted change, is a change of the value of a fluent during a
given interval, without any reason known by the agent; $S$ means that 
surprises may occur (and are thus considered when computing completions of
scenarios). Inertia implies that surprises occur with a low frequency, and
should thus be {\em minimized} in order to be in accordance with observations.
At this point we should say something about actions. In this paper we are not
concerned with actions, and the general assumption we make throughout the paper
is that the agent is totally passive (no action is performed)\footnote{However, the
class ${\cal K}-IS$ does not exclude actions, but assumes that actions have
deterministic and unconditional results, have only instantaneous effects, and that 
there is no
dependency between fluents (which avoids the ramification problem). 
Under these assumptions, it is not hard to see that a deterministic, 
unconditional action can be seen as an observation: for instance, performing
the action ${\tt Load}$ at time point $t$ can be considered equivalent to
observing its deterministic, unconditional result $loaded$ at $t$.}. 

Let ${\scr L}$ be a propositional language built on  a set of propositional variables 
$V=\{v_1,\ldots,v_n\}$. A {\em fluent} is a variable or its negation, 
{\em i.e.}, $f=v_i$ or 
$f=\neg v_i$. The set of all fluents is denoted by $F$; thus $F=\{v_1,\neg v_1, \ldots,
v_n, \neg v_n\}$. Time is assumed to be discrete: the time scale $T$ is a sequence of 
integers
$\{0,1,\ldots,t_{max}\}$ (it is assumed that between two consecutive time-points the system is completely
inert).

If $\varphi \in {\scr L}$ and $t \in T$ then $[t]\varphi$ is an
{\em elementary timed formula} (meaning that $\varphi$ holds at
$t$). (Complex) timed formulas are built from elementary timed formulas and usual connectives.

A {\em timed model} $M$ is a mapping from $V \times T$ to $\{0,1\}$ : $M$ assigns a truth value to each
fluent at each time point. $M \models [t]v_i$ iff $M(v_i,t)=1$. Satisfaction is extended to elementary timed
formulas in the usual way, \emph{i.e.}, $M \models [t]\varphi \wedge \varphi'$ iff $M \models [t]\varphi$ and $M
\models [t]\varphi'$, etc., and then to complex timed formulas: 
$M \models [t]\varphi \wedge [t']\varphi'$ iff $M \models [t]\varphi$ and $M
\models [t']\varphi'$, etc. Note that $[t]\varphi \wedge \varphi'$ is equivalent to $[t]\varphi \wedge
[t]\varphi'$, $[t]\varphi \vee \varphi'$ is equivalent to $[t]\varphi \vee
[t]\varphi'$, etc.

Lastly, we abbreviate $[t]\varphi \wedge [t+1]\varphi \cdots \wedge [t']\varphi$ by $[t,t']\varphi$.

\begin{definition} A {\em scenario} $\Sigma$ in K-IS consists in a set of timed-formulas
\end{definition}


\begin{definition}Let $f \in F$, $t \in T$, $t \neq 0$.
\begin{enumerate}
\item $\langle f,t \rangle$ is a {\em change\/} in $M$
      (notation: $C_{M}(f,t)$), iff $M\models [t{-}1]f\ {\rm and}\ M \models [t]\neg f$.
\item The {\em set of all changes in} $M$ is $C(M) = \{ \langle f,t \rangle \;|\; C_{M}(f,t) \}$.
\end{enumerate}
\end{definition}
When the exact time-point where the change occurred is not known, we affect an interval to a change occurrence and
call it a {\em surprise}.

\begin{definition}
Let $f \in F$, $t,t' \in T$, $t<t'$.
$\langle f,t,t' \rangle$ is a {\em surprise\/} with respect to $M$
      (notation: $S_{M}(f,t,t')$), iff \quad $M\models [t]f\ {\rm and}\ M \not \models [t,t'] f$.
\end{definition}
Intuitively, $S_M(f,t,t')$ means that the truth value of $f$ changed at least once between $t$ and $t'$ (note that it
does not necessarily imply that $ M \models [t']\neg f$, since $f$ may have changed its truth value several times
within $[t,t']$).

If $t'=t+1$, $S_M(f,t,t')$ is equivalent to $C_M(t')$; otherwise ($t'-t>1$) the surprise is said to be {\em
time-ambiguous}.

The following property is straightforward:
\begin{proposition}
$S_M(f,t,t')$ iff $C_M(f,t+1)$ or $C_M(f,t+2)$ or $\ldots$ or $C_M(f,t')$.
\end{proposition}

\subsection{Minimal explanations}
The problem that we address is: given a scenario $\Sigma$, find the preferred 
timed-models 
satisfying
$\Sigma$, \emph{i.e.}, find the timed-models in which the changes are minimal:
$\{M, M \models \Sigma$ and there is no $M'$ such that $C(M') \subsetneq C(M) \}$.

\begin{definition} A {\em pointwise explanation} PE for $\Sigma$ is a set of changes 
$\{\langle f_i,t_i \rangle, i = 1
\cdots p\}$ such that $ \Sigma \;\cup\; \{ [t{-}1]f{\rightarrow} [t]f
        \;|\; f{\in}F,t{\in}T,t{\not=}0,
                \;\langle f,t \rangle \not\in {PE}\}$ is consistent.
\end{definition}
Intuitively, PE consists in specifying which fluents change their value and when so as to be in accordance
with $\Sigma$.

\begin{example}
Let us consider a set of pointwise observations $\Sigma_1$ such that:
\begin{center}
\begin{math}
\Sigma_1: \left\{ \begin{array}{l}
\mbox{[0] }a \wedge b\\
\mbox{[5] }\neg a\\
\end{array}
\right. \end{math}
\end{center}
$\{\langle a,2\rangle \}$ is a pointwise explanation for $\Sigma_1$: in order to explain the non-inert
behaviour of the system, we must assume that the value of the fluent {\em a} changed between 0 and 5, 
(for instance
at time point 2).
\end{example}

Note that $\{\langle a,2\rangle, \langle b,3 \rangle \}$ is also a pointwise explanation, but, since it is not
necessary to assume that $b$ changed in order to explain $\Sigma_1$, this pointwise
explanation is not minimal:

\begin{definition}A pointwise explanation PE is {\em minimal} iff there is no pointwise 
explanation PE' strictly included in $PE$ \end{definition}

\begin{definition} An {\em explanation} for $\Sigma$ is a set of surprises $\{<f_i,t_i,t'_i>,i = 1
\cdots p\}$ such that
$\forall (t''_1,\ldots, t''_p) \in 
[t_1+1,t'_1] \times 
\cdots \times [t_p+1,t'_p]$, $PE=\{\langle f_i,t''_i \rangle, i=1 \cdots p \}$ is a pointwise 
explanation for $\Sigma$. We say that each of these $PE$ is {\em covered} by $E$.

For any two explanations $E$ and $E'$, we say that $E$ \emph{covers} $E'$ iff any pointwise explanation covered by
$E'$ is also covered by $E$.
\end{definition}

Note that an explanation has to be understood as a disjunction of pointwise explanations.
In the example above, $\{\langle a,2,5 \rangle \}$ is an explanation for $\Sigma_1$;
corresponding to the pointwise explanations $\{\langle a, 2\rangle, \langle a,
3\rangle,
\ldots, \langle a, 5 \rangle \}$.

\begin{definition}An explanation $E$ is {\em minimal} if any pointwise explanation which 
is covered by $E$ is minimal.\end{definition}

The previous explanation $\{\langle a,2,5 \rangle\}$ is minimal, 
but for instance 
$\{\langle a,2,5\rangle$, $\langle b,1,3 \rangle\}$
or  $\{\langle a,2,5\rangle, \langle a,0,5 \rangle\}$ are explanations but they are 
not minimal.

\begin{proposition}
$E=\{\langle f_i,t_i,t'_i \rangle, i=1 \ldots p\}$ is an explanation  
(resp. a minimal explanation) for 
$\Sigma$ iff \\
$\Sigma \;\cup\; \{ [t]f{\rightarrow} [t+1]f\;|\; f{\in}F,t{\in}T,t{\not=}t_{max}, 
\mbox{ and }  
\nexists \langle f_i,t_i,t'_i \rangle \in {E} \mbox{ s.t. } t\in [t_i,t'_i-1]\}$
is consistent.

It is a minimal explanation iff the above set of formulas is maximally consistent in 
$\Sigma \;\cup\; \{ [t]f{\rightarrow} [t+1]f
    \;|\; f{\in}F,t{\in}T,t{\not=}t_{max}\}$
\end{proposition}

\begin{definition} A minimal explanation E is compact iff there is no explanation $E'$ 
which strictly covers $E$. 
\end{definition}

In the previous example, the only compact minimal explanation for  $\Sigma_1$ is $\{\langle a,0,5 \rangle \}$.

%
%
%
\begin{proposition}
If there is no disjunction in $\Sigma$ (\emph{i.e.}, each pointwise observation $[t]\varphi$ is a 
conjunction of fluents) there is only one 
compact minimal explanation for $\Sigma$.
\end{proposition}
 
\begin{example} \label{ex2}~
\begin{center}
\begin{math}
 \Sigma_2: \left\{ \begin{array}{l}
\mbox{[0] }a\\
\mbox{[5] }a \vee c\\
\mbox{[10] }b\\
\mbox{[15] }\neg a \vee \neg b\\
\mbox{[20] }\neg c
\end{array}
\right.
\end{math}
\end{center}

There are 3 minimal compact explanations: \\
$E_1=\{\langle a,5,15 \rangle\}$, 
$E_2=\{\langle b,10,15 \rangle\}$ 
and $E_3=\{\langle a,0,5 \rangle,\langle c,5,20 \rangle\}$ 
covering $10+5+(5 \times 15 \rangle$ pointwise 
explanations.

\begin{itemize}
\item $\{\langle a,10,15 \rangle\}$ is an explanation for $\Sigma_2$ but it is not compact 
since $\{\langle a,5,15 \rangle\}$ is an explanation for $\Sigma_2$.
\item $\{\langle a,0,15 \rangle\}$ is not an explanation for $\Sigma_2$ since, for instance, 
$\{\langle a,1 \rangle\}$ 
is not a pointwise explanation for $\Sigma_2$.
\item $\{\langle a,0,15 \rangle,\langle c,5,20 \rangle\}$ is an explanation for $\Sigma_2$ 
but it is not minimal since 
$\{\langle a,6 \rangle,\langle c,7 \rangle\}$ is covered by it and is not a minimal pointwise explanation 
(since 
$\{\langle a,6 \rangle\}$ alone is a pointwise explanation).
\end{itemize}
\end{example}

\begin{definition} Let $Cme(\Sigma)$ be the set of all compact minimal
explanations for $\Sigma$. 
\end{definition}

\begin{proposition}
For each minimal pointwise explanation PE for $\Sigma$, there is a {\em unique} compact
minimal explanation E which covers it.
\end{proposition}
Thus, $Cme(\Sigma)$ is the most concise expression covering all minimal pointwise
explanations.

\begin{proposition}
$Cme(\Sigma)=\{\emptyset\} \Leftrightarrow \Sigma \cup \{[t-1]f \rightarrow [t]f \;|\; 
f \in F,t\in T, t\neq
0\}$ 
is consistent.\footnote{Similarly, in model-based diagnosis, when the system
description and the observations are consistent with the non-failure assumptions,
the preferred diagnosis is $\{\emptyset\}$ (no faulty component).}
\end{proposition}
Intuitively, $Cme(\Sigma)=\{\emptyset\}$ means that assuming that all fluents kept their value throughout T does
not lead to any inconsistency (with $\Sigma$).

\subsection{Computing minimal explanations}
The method consists in completing the set of observations $\Sigma$ by a set of 
persistence assumptions in order to express that, by default, fluents persist during 
the 
interval between two consecutive observations. In order to add only ``relevant'' 
persistence 
assumptions,  we identify, for each variable, the ``relevant time points'' where we have 
some 
observation regarding it.

\begin{definition}
The {\em set of relevant variables} $V(\Sigma)$ 
with respect to the set of observations $\Sigma$ is defined by:
$V(\Sigma)=\{v \in V \;|\; v $ appears in $\Sigma \}$.

The {\em set of relevant time points} $RT_{\Sigma}(v)$ to a variable $v \in V(\Sigma)$ 
w.r.t. $\Sigma$ is defined by: $RT_{\Sigma}(v)=\{t \in T \;|\; [t] \varphi \in \Sigma$ and $v$
appears in $\varphi\}$

Let $RT^*_{\Sigma}(v)$ be the set of relevant time points to a variable $v$ except the last one:
$RT^*_{\Sigma}(v)=RT_{\Sigma}(v) \setminus \{max \{t \in T, t \in RT_{\Sigma}(v)\}\}$
\end{definition}

In example \ref{ex2}: $RT_{\Sigma_2}(a)=\{0,5,15\}$, $RT_{\Sigma_2}(b)=\{10,15\}$ 
and 
$RT_{\Sigma_2}(c)=\{5,20\}$.

%
\begin{definition}
The set of persistence axioms of $\Sigma$ is defined by:  
\[PERS(\Sigma)=\bigcup_{v \in V(\Sigma), t \in RT^*_{\Sigma}(v)} \{[t]v \rightarrow [next(v,t)]v,\quad
[t]\neg v \rightarrow [next(v,t)]\neg v\}\]

where $next(v,t)=\min \{t' \in RT_{\Sigma}(v), t'>t \}$ is the 
next relevant time point for $v$ after $t$.
\end{definition}

%


\begin{definition}~
\begin{itemize}
\item Let $MaxCons(\Sigma)$ be the set of all maximal subsets of $PERS(\Sigma)$ consistent with $\Sigma$
\item let $Cmc(\Sigma)=\bigcup_{X \in MaxCons(\Sigma)} (PERS(\Sigma) \setminus X)$
\item let $Candidates(\Sigma)=${\large \{}$ \{ \langle f, t,t' \rangle \;|\; [t]f \rightarrow [t']f \in
X\}, X \in Cmc(\Sigma) ${\large \}}.
\end{itemize}
\end{definition}

\begin{proposition}~
\begin{itemize}
\item Any set of surprises in $Candidates(\Sigma)$ is a minimal explanation for $\Sigma$.
\item Every pointwise explanation for $\Sigma$ is covered by a unique set of surprise in $Candidates(\Sigma)$
\end{itemize}
\end{proposition}

This principle of finding minimal explanations corresponds exactly to finding minimal sets
of faulty components in model-based diagnosis (or candidates in ATMS \cite{deKleer86}
terminology). Hence, the computation of minimal explanation can be done by well-known
algorithms from these fields.

\begin{example}(Example \ref{ex2} continued)
The relevant time-points for $\Sigma_2$ are :\\
$RT_{\Sigma_2}(a)=\{0,5,15\}$, $RT_{\Sigma_2}(b)=\{10,15\}$ and 
$RT_{\Sigma_2}(c)=\{5,20\}$.
\begin{center}
\begin{math}
PERS(\Sigma_2): \left\{ \begin{array}{l}
\mbox{[0]}a \rightarrow [5]a\\
\mbox{[5]}a \rightarrow [15]a\\
\mbox{[0]}\neg a \rightarrow [5]\neg a\\
\mbox{[5]}\neg a \rightarrow [15]\neg a\\
\mbox{[10]}b \rightarrow [15]b\\
\mbox{[10]}\neg b \rightarrow [15]\neg b\\
\mbox{[5]}c \rightarrow [20]c\\
\mbox{[5]}\neg c \rightarrow [20]\neg c\\
\end{array}
\right.
\end{math}
\end{center}
When computing the minimal set of persistence assumptions which has to be
removed from $PERS$ in order to 
make it consistent with $\Sigma_2$, we obtain the 3 minimal candidates 
$\{\langle a,5,15\rangle\}$, $\{\langle b, 10,15\rangle\}$ and $\{ \langle a,0,5 \rangle, 
\langle c,5,20 \rangle \}$.

Note that it is not always the case that the computed explanations are compact:
consider $\Sigma_3 =\{[0]a, [5]a \vee b, [10]\neg a \}$. We find the two minimal
explanation $\{\langle a,0,5 \rangle\}$ and $\{\langle a,5,10 \rangle\}$ which are not 
compact (since they can be compacted into $\{\langle a,0,10\rangle\}$).
\end{example}

If we wish to  obtain the {\em compact} minimal explanations from $Candidates(\Sigma)$, we
may define a compactification procedure, where precise description is outside the scope of
the paper. Intuitively, this procedure consists in compacting explanations pairwise and
iterate until no pair of explanations is compactable. This operation is confluent and its
result is exactly $Cme(\Sigma)$.

\section{Probabilities and unpredicted change}

We assume now that probabilistic information about fluents truth values and 
persistence is 
available. This will enable us to 
compute the probability of each explanation, and thus will help us ranking 
them. We will take account of 2 important factors which can influence quantitatively the 
persistence: time duration and the intrinsic tendency of each fluent to persist (some fluents 
persist longer than other, e.g., $sleeping$ persists usually longer than {\em eating an apple}).

\subsection{Markovian fluents}
 For the sake of simplicity, 
throughout this Section we assume 
fluents are {\em mutually independent} and {\em Markovian}\footnote{
These assumptions mean that timed-variables can be structured in a very simple 
temporal Bayesian network, which has the following form (where  $\{v_1, \ldots, v_n\} 
\in V$ are the propositional variables):\\
$[0]v_1 \rightarrow [1]v_1 \rightarrow \cdots \rightarrow [t_{max}]v_1$\\
$[0]v_2 \rightarrow [1]v_2 \rightarrow \cdots \rightarrow [t_{max}]v_2$\\
$\vdots$\\
$[0]v_n \rightarrow [1]v_n \rightarrow \cdots \rightarrow [t_{max}]v_n$ (Dean
and Kanazawa 1989 \cite{DeKa89}).}.
We will consider 
stationary fluents in some cases but we will not impose this restriction to the whole 
section.  
\begin{definition}~
\begin{itemize}
\item  fluents are {\em mutually independent} iff $\forall t,t', \forall f \in F,
 \forall f' \in 
F \setminus \{f,\neg f\},$
\begin{center} $Pr([t]f \wedge [t']f')=Pr([t]f).Pr([t']f')$.\end{center} 
\item  f is {\em stationary} iff $\forall t,t', \forall f \in F, \quad Pr([t]f)=Pr([t']f)=p_f$. 
\item  f is {\em Markovian} iff $\forall t, \quad Pr([t+1] f \;|\; [t] f \wedge 
H_{0 \rightarrow t-1} )
= Pr([t+1]f \;|\;[t]f) = 1 - 
\varepsilon_f$.

Where $H_{0 \rightarrow t-1}$ is the history of the system from 0 to $t-1$.
\end{itemize}
\end{definition}
The independence and Markovian assumption imply that the only necessary data are, for 
each propositional 
variable $v$ and each time-point $t$, a {\em prior probability} of $v$, $p_{t,v}$ and 
the {\em elementary 
change probabilities} 
$\varepsilon_v$ and $\varepsilon_{\neg v}$. We will show later that in many cases 
prior probabilities are not necessary. Generally $\varepsilon_v \neq \varepsilon_{\neg v}$,
\emph{i.e.}, the persistence of a fluent may 
be different from the persistence of its negation (think of $f=$alive or 
$f=$bell-ringing).

If fluents are stationary then for each variable it is enough to consider a prior probability 
$p_v$ which is independent of the time-point. Moreover, in this case, there is a 
relationship between the persistence of a fluent and its negation:
\begin{proposition} If $f$ is stationary then 
$\varepsilon_f. p_f= \varepsilon_{\neg f}. p_{\neg f}$
\label{stationnaire}\end{proposition}
\begin{proof} $Pr([t]f \wedge [t+1]\neg f) = Pr([t+1] \neg f \;|\; [t] f). 
Pr([t]f)= \varepsilon_f.p_f$. And $Pr([t]\neg f \vee [t+1] f) = Pr([t] 
\neg f ) + Pr([t+1]f) -Pr([t]\neg f \wedge [t+1]f) = p_{\neg f} + p_f -\varepsilon_{\neg 
f}.p_{\neg f}=1-\varepsilon_{\neg_f}.p_{\neg f}$.
\end{proof}

Obviously, some fluents are 
{\em not} Markovian:
\begin{itemize}
\item either because they do not tend to persist independently of t: consider a clock which 
always rings around 7 a.m., then the fluent $ringing$ is 
not Markovian (nor stationary). However, the fact that the clock is on the 
bedside table is usually a Markovian fluent (unless you throw it away each time it rings ).
\item or because their tendency to persist depend on their history: consider a bus which 
period is 15 minutes,  the probability that the bus will come depends on the length of the 
interval since it last came. If you know that no bus has came for 10 minutes, you have 
more chance than the bus will arrive immediately than if you know that no bus has come 
for 5 minutes. Clearly, $Pr([t] \neg bus\_coming \;|\; [t-1] \neg 
bus\_coming)=\frac{14}{15}$, while $Pr([t]bus\_coming \;|\; [t-k,t-1]\neg 
bus\_coming)=\frac{k}{15}$. And thus \emph{bus\_coming} is not Markovian as all 
periodic fluents (however it 
may reasonably be assumed stationary without considering the rush hour) .
\end{itemize}
Among all Markovian fluents, some can be distinguished:
\begin{itemize}
\item persistent fluents, such that $\varepsilon_f=0$ (e.g. $f=$dead). A stationary and 
persistent fluent is a degenerate case (either it never change, or the probability of its 
opposite is null).
\item chaotic fluents, such that: 
$Pr([t+1]f\;|\;[t]f)=Pr([t+1]f)$ (these 
fluents have no tendency to persist, since knowing this fluent true at $t$ does not make it 
true at $t+1$ more probable). It means that $\varepsilon_f=1-p_{t+1,f}$, hence, chaotic 
fluents must be stationary.
\item switching fluents (provided that the time unit is well chosen) such that 
$\varepsilon_f=1$
\end{itemize}

Now, given the prior probability distribution $p_f$ and the switch probability \\
$Pr([t+1]\neg f \;|\; [t]f)= 
\varepsilon_f$ (which are constants, independent of t), it is possible 
to 
compute the prior probability of surprise for each stationary Markovian fluent between 
any 
time points.

\begin{definition}
Notation: $Pr(\langle f,t,t' \rangle)=Pr([t]f \wedge \neg [t,t']f)$ is the probability of the
surprise $\langle f,t,t' \rangle$.
\end{definition}

\begin{proposition} If f is stationary and Markovian, then 
\[Pr(\langle f,t,t+n \rangle) =(1- (1-\varepsilon_f)^n ).p_f\] \label{Stt+n}
\end{proposition}

\begin{proof}
\begin{eqnarray*}
Pr(\langle f,t,t+n \rangle) &= & Pr(\langle f,t,t+n \rangle)\;|\;[t] f). 
Pr([t]f) \\
&&+ \hspace{0.5cm} Pr(\langle f,t,t+n \rangle)\;|\; [t] \neg f). Pr([t] \neg f)\\
&=& (1-Pr([t]\neg f \vee [t+1,t+n]f \;|\; [t]f)).p_f + 0 \\
&=& (1-Pr([t+2,t+n]f \;|\; [t+1]f).Pr([t+1]f \;|\; [t]f)).p_f\\
&=&(1-(1-\varepsilon_f)^n).p_f
\end{eqnarray*}
\end{proof}
 
Indeed, it looks natural that, 
knowing the truth value of a fluent at a given instant, the more time has passed, 
the 
more likely the fluent had its truth value changed.

%
%
%
%

\subsection{Highly persistent fluents and probabilities of explanations}

From proposition \ref{Stt+n} and the independence assumption it is also possible to 
compute the prior probability of an explanation (minimal or not) :
\begin{center} $Pr(\{\langle f_i,t_i,t'_i \rangle, 
i=1\ldots n\})=\prod_{i=1}^n(1-(1-\varepsilon_{f_i})^{t'_i-t_i}).p_{f_i}$\end{center}
Now, these prior probabilities have to be conditioned on the observations ($\Sigma$) in 
order to get posterior probabilities of all explanations : if E is an explanation, then $Pr(E\;|\; 
\Sigma)=\frac{Pr(E \wedge \Sigma)}{Pr(\Sigma)}$.

In the general case these probabilities are not straightforward to compute, mainly because 
$Pr([t]f \wedge [t']\neg f)$ may be much lower than
$Pr(\langle f,t,t' \rangle)$\footnote{Indeed $\langle f,t,t' \rangle$ is true if f holds at t and changed
its value {\em at least once} between t and $t'$, while $[t]f$ and $[t']\neg f$ is
true if f changed its value {\em an odd number of times} between t and $t'$.}. 
Moreover, non-minimal explanations may have a significant probability, 
even in the case where $\{\emptyset \}$ is an explanation, which contradicts the principle 
of minimal change. This is because elementary probabilities of change (the 
$\varepsilon_f$'s) may be large. Computing probabilities of change over an interval 
($Pr([t]f \wedge [t']\neg f)$) requires computing the probability that an odd number 
of pointwise surprises occurred between t and $t'$. This is left outside of the scope of the 
paper which focuses on infrequent change (thus to be minimized). As expected intuitively, change is minimized if and only 
if the $\varepsilon_f$'s are all infinitely small; more precisely, if for any considered 
interval $[t,t']$, $(t-t')\varepsilon_f \ll 1$.

Informally,
f is {\em highly persistent} w.r.t. interval $[t,t']$ iff $\varepsilon_f \ll \frac{1}{t'-t}$.

\begin{proposition}\label{prop26} Assume that all fluents are highly persistent w.r.t. $[0,t_{max}]$, then:
\begin{enumerate}
\item $Pr(\langle f,t,t' \rangle) \approx Pr([t]f \wedge [t']\neg f) = (t'-t)\varepsilon_f. p_f + 
o(\varepsilon_f)$. 
\label{approx}
\item $\sum_{E \in Cme(\Sigma)}Pr(E\;|\;\Sigma) \approx_{\varepsilon_f \rightarrow 0, 
\forall f} 1$. \label{somme1}
\end{enumerate}
\end{proposition}

Item \ref{prop26}.\ref{somme1} means that non-minimal explanations for $\Sigma$ have a 
very low posterior probability and thus they can be neglected. In particular, if 
$\{\emptyset\}$ is an explanation for $\Sigma$ then $Pr(\{\emptyset\}\;|\;\Sigma) \approx 
1$. Proposition \ref{prop26}.\ref{approx} intuitively means that if $f$ changed at least once its truth 
value within $[t,t']$ ($\langle f,t,t' \rangle$) then it changed exactly once (and thus $\neg f$ holds at 
t').
Now probabilities of minimal explanations are easy to compute (provided that all fluents 
are highly persistent). We start by giving two detailed examples.

\begin{example}(cat A and B, pure prediction)
\begin{center}
\begin{math}
 \Sigma_4: \left\{ \begin{array}{l}
\mbox{[0] } a\\
\mbox{[15] } b\\
\mbox{[20] } \neg a \vee \neg b\\
\end{array}
\right.
\end{math}
\end{center}
This means that at 6h, cat A 
is sleeping inside the house, at 6h15, cat B is also sleeping in the house, and at 6h20, 
one cat is not sleeping inside the house. \\
$Pr(\{\langle a,0,20 \rangle\}) \approx 20 \varepsilon_{a}p_{a}$, 
$Pr(\{\langle b,15,20 \rangle\})\approx 5  \varepsilon_{b}p_{b}$.
\begin{eqnarray*}
Pr(\{\langle a,0,20 \rangle\;|\;\Sigma_4\}) &=&
\frac {Pr(\langle a,0,20 \rangle \wedge \Sigma_4)}{Pr(\Sigma_4)}\\ 
Pr(\langle a,0,20 \rangle \wedge \Sigma_4) &=& 
Pr([0]a \wedge [20] \neg a \wedge [15]b)\\
&=&Pr([0]a \wedge [20] \neg a ).Pr([15]b) \\
&\approx& 20 \varepsilon_a .p_a .p_b.
\end{eqnarray*} 
Similarly, 
$Pr(\{\langle b,15,20 \rangle\;|\;\Sigma_4\})  
\approx 5 \varepsilon_b .p_b .p_a $
\begin{eqnarray*}
Pr(\Sigma_4) &= & Pr([0]a \wedge [15]b \wedge 
[20] \neg a \vee \neg b)\\
&=&Pr([0]a \wedge [15]b \wedge [20] \neg a) +Pr([0]a \wedge [15]b 
\wedge [20]  \neg b) \\
&&- Pr([0]a \wedge [15]b \wedge [20] \neg a \wedge [20] \neg b) \\
&\approx & 20 \varepsilon_a .p_a .p_b +5 \varepsilon_b .p_b .p_a - 
(20 \varepsilon_a .p_a .p_b)(5 \varepsilon_b .p_b .p_a)\\
&&\mbox{(the last term being negligible)}\\
&\approx & Pr(\langle a,0,20 \rangle \wedge \Sigma_4)+
Pr(\langle b,15,20 \rangle  \wedge \Sigma_4)
\end{eqnarray*}
$Pr(\langle a,0,20 \rangle \;|\;\Sigma_4) \approx 
\frac{20 \varepsilon_a .p_a .p_b}
{\varepsilon_a .p_a .p_b +5 \varepsilon_b .p_b .p_a}
=\frac{1}
{1+\frac{\varepsilon_b}{4.\varepsilon_a}}$\\ 
and $Pr(\langle b,15,20 \rangle\;|\;\Sigma_4) 
\approx \frac{1}{1+\frac{4.\varepsilon_a}{ \varepsilon_b}}$.

Remarkably, these probabilities do not depend on the prior probabilities $p_a$ and 
$p_b$, which is natural (since $a$ and $b$ are known to hold at 0 and 15, respectively); 
this will always be the case for pure prediction problems where disjunctive 
observations appear only at the last time point.
Note that, if $\varepsilon_a=\varepsilon_b$, we get $Pr(\langle a,0,20 \rangle=\frac{4}{5}$ and 
$Pr(\langle b,0,20 \rangle)=\frac{1}{5}$.\footnote{More generally, for $\Sigma =\{[t_a]a,[t_b]b, [t']\neg a \vee \neg b\}$ and 
$\varepsilon_a=\varepsilon_b$, we get  
$Pr(\langle a,t_a,t' \rangle=\frac{t'-t_a}{(t'-t_a) + (t'-t_b)}$ 
and $Pr(\langle b,t_b,t'\rangle=\frac{t'-t_b}{(t'-t_a) + (t'-t_b)}$.}
\end{example}

\begin{example}(pure postdiction)
\begin{center}
\begin{math}
 \Sigma_5: \left\{ \begin{array}{l}
\mbox{[0] } a \vee b\\
\mbox{[5] } \neg a\\
\mbox{[20] } \neg b\\
\end{array}
\right.
\end{math}
\end{center} 

\begin{eqnarray*}
Pr(\langle a,0,5 \rangle \wedge \Sigma_5) &= &
Pr([0]a \wedge [5]\neg a \wedge [20]\neg b) \approx 
5\varepsilon_a.p_a.(1-p_b).\\ 
\mbox{Similarly, } 
Pr(\langle b,0,20 \rangle \wedge \Sigma_5) & \approx& 20 
\varepsilon_b.p_b.(1-p_a).\\
Pr(\Sigma_5) &\approx & Pr(\langle a,0,5 \rangle \wedge \Sigma_5) 
+ 
Pr(\langle b,0,20 \rangle \wedge \Sigma_5).\\ 
Pr(\langle a,0,5 \rangle \;|\;\Sigma_5) & =&
\frac{5\varepsilon_a.p_a.(1-
p_b)}{5\varepsilon_a.p_a.(1-p_b)+20 \varepsilon_b.p_b.(1-p_a)}.
\end{eqnarray*} If we suppose the 
fluents stationary (it is natural to do so considering the usual behaviour of cats) then 
using proposition 
\ref{stationnaire} we obtain the same kind of results as in the previous case : 
$Pr(\langle a,0,5 \rangle\;|\;\Sigma_5) =1/5$ and $Pr(\langle b,0,20 \rangle \;|\;\Sigma_5)=4/5$.
\end{example}

\begin{example}
Take $\Sigma_2$ of example \ref{ex2}, we have:\\
$Pr(\langle a,5,15 \rangle) \;|\;
\Sigma_2)=\frac{1}{1+\frac{\varepsilon_b}{2 \varepsilon_a}+
\frac{15}{2}\varepsilon_c.\frac{p_c}{1-p_c}}$\\
$Pr(\langle b,10,15 \rangle \;|\; \Sigma_2)=\frac{1}{1+2\frac{\varepsilon_a}{ \varepsilon_b}+
15\frac{\varepsilon_a.\varepsilon_c}{\varepsilon_b}.\frac{p_c}{1-p_c}}$\\
and $Pr(\langle a,0,5 \rangle \wedge \langle c,5,20 \rangle\;|\; \Sigma_2)=\frac{1}
{1+\frac{2 (1-pc)}{15 (\varepsilon_c)}+
\frac{\varepsilon_b}{15\varepsilon_a.\varepsilon_c}.\frac{1-p_c}{p_c}}$

Note that if $\varepsilon_a=\varepsilon_b=\varepsilon_c$ and are very small,
$Pr(\langle a,5,15 \rangle \;|\;
\Sigma_2)\approx \frac{2}{3}$, $Pr(\langle b,10,15 \rangle \;|\; \Sigma_2)
\approx \frac{1}{3}$ and $Pr(\langle a,0,5 \rangle 
\wedge \langle c,5,20 \rangle\;|\; \Sigma_2)\approx 0$. 
These results are independent of the prior
probabilities $p_a$, $p_b$, $p_c$.
\end{example}

\section{Discussion}

The original aspects of our work are that (i) minimal changes (explanations)
are provided together with the interval when change may have occurred, which
makes the representation concise, and (ii) knowing probabilities of change
of fluents from one point to the subsequent one, and fluent prior probabilities,
we compute the probability of each explanation and thus we rank them
accordingly.
We showed that minimizing change is coherent with a probabilistic handling of 
change only if probabilities of change are very small (with this assumption,
only minimal explanations may have significant probabilities). This probabilistic
handling of change needs explicit and metric time, and is coherent with the
intuitive idea that the longer the time interval during which a change may have
occurred, the more likely it has actually occurred within this interval.
Many related works share some features with ours. The first approach dealing with
unpredicted change (``fluents that may change by themselves'') is Lifschitz and Rabinov's
\cite{LiRa89}; using the situation calculus and circumscription, they minimize the set
of unpredicted changes from one state to another, giving, thus, something analogous to our
pointwise explanations (without the explicit temporal dimension). Our notion of
``surprise'' is borrowed from Sandewall's ontology
\cite{Sandewall94} where surprises are
defined as unpredicted changes with low frequency; in an earlier version of his book
\cite{Sandewall92}, he proposed to select among competing surprise sets by associating to
each fluent a penalty (the higher the penalty, the more unlikely the fluent may change
unpredictedly), preferred surprises sets minimizing the sum of the penalties of changing
fluents. Dealing with penalties is very analogous to dealing with probabilities,
especially in the case of infinitely small probabilities \cite{DLS94b}, thus our
approach extends his by taking interval durations into account in defining the penalties. 

A probabilistic handling of unpredicted change using explicit, metric time has been 
proposed for temporal projection by Dean and Kanazawa \cite{DeKa89} and later 
on by Hanks and McDermott \cite{HaMD94} in a more general setting. 
In \cite{DeKa89}, unpredicted
change is modelled with exponential survivor functions, asserting that 
$Pr([t+ \Delta t] f \;|\; [t] f) = \exp(\frac{-\Delta t}{\lambda_f})$. 
They do not assume that probabilities of
change are small, and thus do not minimize change. We recall that the reason why we want
change to be minimized is that it leads to a concise list of explanations, since only
minimal explanations may have significant probabilities. This has to
be related to the fact that, in \cite{DeKa89}, probabilities of change are only
used 
for temporal projection, namely, computing the probability of given fluents at the 
latest time point, but they are not used to deal with explaining scenarios 
and  the reasoning is only
forward (it has for instance no postdiction capabilities). Lastly,
\cite{DeKa89} framework
assumes that observations are atomic (no disjunctive observations are allowed while
our approach allows the observation of any propositional formula). The same remarks
apply to \cite{HaMD94}. 

Other related approaches include work in the domains of probabilistic abduction and
model-based diagnosis, temporal diagnosis, and also belief update. Probabilistic (and
cost-based) abduction (\cite{Poole91}, \cite{EiGo93}, \cite{deKlWi87}) 
attach probabilities to explanations from prior probabilities of hypotheses (faults
in model-based diagnosis) and generally an independence assumption amongst them. 
In the domain of temporal model-based diagnosis, Friedrich and Lackinger
\cite{FrLa91} attach time intervals to fault configurations, meaning that
a given component is in a failure mode during the whole interval -- contrarily to our
surprises which are instantaneous (recall that the interval in a surprise is understood
disjunctively, not conjunctively); probabilities are attached to temporal
diagnoses, the evolution of the system is modelled by Markov chains (with transition 
probabilities for the different modes of a component). 
Console et al. \cite{Console92} attach time points to fault configurations; 
their approach is extended in a probabilistic setting by Portinale \cite{Portinale92} 
who also models the evolution of the system by a Markov process.
Cordier and Thi\'ebaux's event-based diagnosis \cite{CoTh94} consider
sequences of events and compute probabilities over diagnoses from priori probabilities
of events (without metric time). Boutilier's generalised update operators
\cite{Boutilier95} handles unpredicted change, explaining observations by events, ranked
according to their plausibility, but without explicit time (he only considers two time
points). If we specialize our framework to a time scale with only two time points
$t_{before}$ and $t_{now}$, then it is possible to show that we obtain something very
close to a generalized update operator, where the allowed events are the elementary
surprises. 

Further work will include the handling of dependent fluents (speciality $D$ in 
Sandewall's ontology) -- in the probabilistic case they may for instance be 
structured in a Bayesian network -- and the integration of surprise minimisation 
together with a handling of actions with alternative effects (speciality $A$ in 
Sandewall's ontology). 

\section*{Acknowledgement}
We would like to thank Didier Dubois and Henri Prade for their comments which helped 
improve this paper.
\baselineskip 0.4cm
\bibliography{flo,Rauc}
\end{document}